\documentclass[12pt,onecolumn,journal,draftclsnofoot,singlespace]{IEEEtran}
\makeatletter
\def\ps@headings{%
\def\@oddhead{\mbox{}\scriptsize\rightmark \hfil \thepage}%
\def\@evenhead{\scriptsize\thepage \hfil \leftmark\mbox{}}%
\def\@oddfoot{}%
\def\@evenfoot{}}
\makeatother 
\usepackage{algorithm}
\usepackage{algpseudocode}
\usepackage{bbm}
\usepackage{times,amsmath,amssymb,cite,color,bbm,array,url}
\usepackage{amsthm}
\usepackage{graphicx,psfrag}
\usepackage{epsfig,epsf}
\ifCLASSINFOpdf

\usepackage{psfrag}
\usepackage[dvips]{graphicx,color}

\else

\fi

\def\boxit#1{\vbox{\hrule\hbox{\vrule\kern3pt
        \vbox{\kern3pt#1\kern3pt}\kern3pt\vrule}\hrule}}

\def\reals{ { {\rm  I \kern-0.15em R }  } }
\def\complex{ {\,{{\rm C} \kern-0.50em \raise0.20ex {  |}}\, }}

\def\Rbf{{\bf R}}

\def\be{\begin{equation}}
\def\ee{\end{equation}}

\def\scalefig#1{\epsfxsize #1\textwidth}
\def\Rxx{\Rbf_{\ssstyle X\kern-.1em X}}

\let\ssstyle=\scriptscriptstyle

\def\etal{{\it et al. \/}}

\def\Kout{\setbox1=\hbox{\Huge\bf K}\hbox to
1.05\wd1{\hspace{.05\wd1}
}}
\def\Sout{\setbox1=\hbox{\Huge\bf S}\hbox to 1.05\wd1{\hspace{.05\wd1}
}}

\def\scalefig#1{\epsfxsize #1\textwidth}
\def\nn{{\nonumber}}

\newtheorem{lemma}{Lemma}
\newtheorem{theorem}{Theorem}

\usepackage{fancyhdr}
\usepackage{amsmath}
\usepackage{authblk}

\title{Anomaly Detection in Hierarchical Data Streams under Unknown Models}
\author[1]{\small Sattar Vakili}
\author[1]{\small Qing Zhao}
\author[2]{\small Chang Liu}
\author[2]{\small Chen-Nee Chuah}
\affil[1]{\small School of Electrical and Computer Engineering, Cornell University, {\{sv388,qz16\}@cornell.edu }}
\affil[2]{Electrical and Computer Engineering Department, University of California, Davis, \{chuah,cchliu\}@ucdavis.edu}

\begin{document}




\maketitle

\begin{abstract}

We consider the problem of detecting a few targets among a large number of hierarchical data streams. The data streams are modeled as random processes with unknown and potentially heavy-tailed distributions. The objective is an active inference strategy that determines, sequentially, which data stream to collect samples from in order to minimize the sample complexity under a reliability constraint. We propose an active inference strategy that induces a biased random walk on the tree-structured hierarchy based on confidence bounds of sample statistics. We then establish its order optimality in terms of both the size of the search space (i.e., the number of data streams) and the reliability requirement. The results find applications in hierarchical heavy hitter detection, noisy group testing, and adaptive sampling for active learning, classification,  and stochastic root finding.

\end{abstract}

%

\section{Introduction}\label{Sec:Intro}

We consider the problem of detecting a few targets with abnormally high mean values among a large number of data streams. Each data stream is modeled as a stochastic process with an unknown and potentially heavy-tailed distributions. The stochastic nature of the data streams may be due to the inherent randomness of the underlying phenomenon or the noisy response of the measuring process. The number of targets is unknown. The objective is to identify all targets (if any) or declare there is no target to meet a required detection accuracy with a minimum number of samples.

Inherent in a number of applications (see Sec.~\ref{appl}) is a tree-structured hierarchy among the large number of data streams. With each node representing a data stream, the tree structure encodes the following relationship: the abnormal mean of a target leads to an abnormal mean in every ancestor of the target (i.e., every node on the shortest path from this target to the root of the tree).  We illustrate in Fig.~\ref{Fig:BT} a special case of a binary-tree hierarchy.

Targets are of two types: leaf-level targets and hierarchical targets. A leaf-level target is a leaf (level $0$) of the tree whose mean is above a given threshold. Hierarchical targets are defined recursively in an ascending order of the level of the tree.  Specifically, an upper-level node with an anomalous mean is a (hierarchical) target if its mean remains above a given threshold after excluding all its target descendants (if any). Otherwise, this upper-level node is only a reflecting point for merely being an ancestor of a target (see Fig.~\ref{Fig:BT}).

The objective of the problem is to detect all targets quickly and reliably by fully exploiting the hierarchical structure of the data streams. Specifically, we seek an active inference strategy
that determines, sequentially, which node on the tree to probe and when to terminate the search
in order to minimize the sample complexity for a
given level of detection reliability. We are particularly interested in strategies that achieve a sublinear scaling of the sample
complexity with respect to the number of data streams. In other words, accurate detection
can be achieved by examining only a diminishing fraction of the search space as the search space
grows.

\begin{figure}[htbp]
\centering
\begin{psfrags}
\psfrag{a}[c]{\small$l=3$}
\psfrag{b}[c]{\small$l=2 $}
\psfrag{c}[c]{\small$ l=1$}
\psfrag{d}[c]{\small$ l=0$}
\psfrag{e}[c]{\scriptsize$(1,0)$}
\psfrag{f}[c]{\scriptsize$(2,0)$}
\psfrag{g}[c]{\scriptsize$(3,0)$}
\psfrag{h}[c]{\scriptsize$ (4,0)$}
\psfrag{i}[c]{\scriptsize$(5,0) $}
\psfrag{j}[c]{\scriptsize$(6,0) $}
\psfrag{k}[c]{\scriptsize$(7,0) $}
\psfrag{l}[c]{\scriptsize$(8,0) $}
\psfrag{m}[c]{\scriptsize$(1,1) $}
\psfrag{n}[c]{\scriptsize$(2,1) $}
\psfrag{o}[c]{\scriptsize$(3,1) $}
\psfrag{p}[c]{\scriptsize$(4,1) $}
\psfrag{q}[c]{\scriptsize$(1,2) $}
\psfrag{r}[c]{\scriptsize$(2,2) $}
\psfrag{s}[c]{\scriptsize$(1,3) $}
\psfrag{z}[c]{\scriptsize$(1,3) $}
\scalefig{0.8}\epsfbox{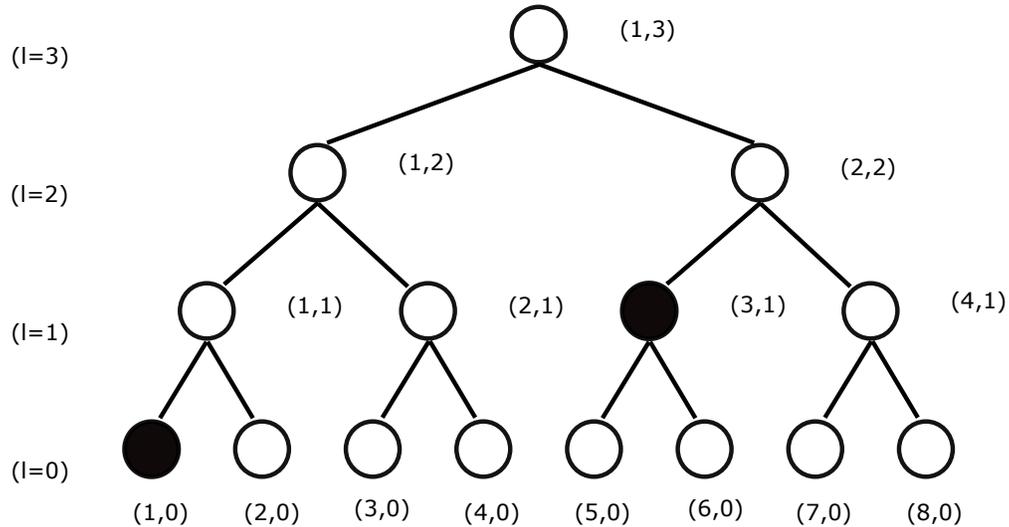}\caption{A tree-structured hierarchy with $l$ denoting the level of the tree and $(k,l)$ the $k$th node on the $l$-th level ($(1,0)$ is a leaf-level target, $(3,1)$ is a hierarchical target. Nodes $(1,1)$, $(1,2)$, $(1,3)$, and $(2,2)$ are reflecting points).
 .} \label{Fig:BT}
\end{psfrags}

\end{figure}

\subsection{Applications}\label{appl}

The above general problem of detecting abnormal mean values in hierarchical data streams arises in a number of active inference and learning applications in networking and data analytics. We give below several representative examples to make the formulation of the problem concrete.

{\bf{Heavy hitter and hierarchical heavy hitter detection:}}~In Internet and other communication and financial networks, it is a common observation that a small
number of flows, referred to as heavy hitters (HH), account for the most of the total traffic~\cite{Thompson}. Quickly identifying the heavy hitters is thus crucial to network
stability and security.
With limited sampling resources at the router, however,  maintaining a packet count of each individual flow is highly inefficient, if
not infeasible. The key to an efficient solution is to consider prefix aggregation based on the source or destination IP addresses. This naturally leads to a binary tree structure with all targets (HHs) at the leaf level.

A more complex version of the problem is hierarchical
heavy hitter (HHH) detection, in which the search for flows with abnormal volume extends to aggregated flows. In other words, there exist hierarchical targets. HHH detection is of particular interest in detecting \emph{distributed} denial-of-service attacks \cite{HHH1}.

%
%
%

{\bf{Noisy group testing:}}~In group testing, the objective is
to identify a few defective items in a large population
by performing tests on subsets of items. Each group test gives a binary
outcome, indicating whether the tested group contains any defective
items. The problem was first motivated by the blood-test screening of
draftees during World War II, for which Robert Dorfman originated the
idea of testing blood samples pooled from a group of people~\cite{Dorfman}.  Since then, the problem has found a wide range of applications, including idle channel detection~\cite{channel}, network tomography~\cite{tomo}, detecting malicious users and attackers~\cite{mali1,mali2}, and DNA sequencing and
screening\cite{DNA1,DNA2}.

Most work on group testing assumes error-free test outcomes (see Sec.~\ref{sec:RelatedWork} for a more detailed
discussion on existing work). The problem studied in this paper includes,
as a special case, adaptive group testing under general and unknown noise
models for the test outcomes. Specifically, the outcome of a group test is
no longer a deterministic binary value, but rather a Bernoulli random variable
with an unknown parameter that represents the false alarm probability (when the
tested group contains no defective items) or the detection power (i.e., the probability of correct
detection when the test group contains defective items). A data stream
thus corresponds to the noisy Bernoulli test outcomes of a given subset
of the population. The hierarchy of the data streams follows from the ``subset'' relationship among the corresponding test groups. Under the practical assumption that false alarm
and miss detection probabilities are smaller than $1/2$, targets are
those leaf nodes whose mean value exceeds $1/2$. A more detailed mapping of the noisy group testing problem
to the active inference problem studied in this paper is given in Sec.~\ref{sec:AL}.

{\bf{Adaptive sampling with noisy response:}}~The problem also applies to adaptive sampling with noisy
response for estimating a step function in $[0,1]$. Such problems arise in active learning of binary threshold classifiers for document classification~\cite{DocClass} and stochastic root finding~\cite{Frazier}.

%

Partitioning
the $[0,1]$ interval into small intervals and sampling the boundary points of each interval, we can map the adaptive sampling problem to the target search problem
where the target is the small interval containing
the location of the step. Examining larger intervals
(consisting of several smaller intervals) induces a hierarchical structure of the noisy responses. See Sec.~\ref{sec:AL} for a detailed discussion.

\subsection{Main Results}

We develop an active inference strategy for detecting an unknown number of targets among a large number $N$ of data streams with unknown distributions. The performance measure is the number of samples (i.e., detection delay) required for achieving a confidence level above $1-\epsilon$ (i.e., the probability that the declared target set does not equal to the true set is bounded by $\epsilon$). By fully exploiting the tree-structured hierarchy, the proposed active inference strategy has a
sample complexity that is order optimal in both the size $N$ of the search
space and the reliability constraint $\epsilon$.

Referred to as Confidence
Bounds based Random Walk (CBRW), the proposed strategy consists of a global random walk on the tree interwoven with a local confidence-bound based test. Specifically, it induces a biased random walk that initiates at the root of the tree and
eventually arrives and terminates at a target with the
required reliability. Each move in the random walk is guided by the output
of a local confidence-bound based sequential test carried on each child of the node currently being visited by the random walk.
This local sequential test module ensures that
the global random walk is more likely to
move toward the target than move away from it and that the random walk terminates at a true target
with a sufficiently high probability.
%
%

The sample complexity of CBRW is analyzed using properties of biased random walk on a tree and large
deviation results on the concentration of the sample mean statistic.
We show that the sample complexity of CBRW is in the order of $O(\log N +\log \frac{1}{\epsilon})$ provided that the gap between the mean value of each data stream and the given threshold is bounded away from $0$.
It is thus order optimal in both $N$ and $\epsilon$ as determined by information-theoretic lower bounds.
Of particular significance is that the effect on the sample complexity from an enlarged search
space (increasing $N$) and an enhanced reliability (decreasing $\epsilon$) is additive
rather than multiplicative. This results from the random walk structure which effectively
separates two objectives of moving to the targets with $O(\log N)$ samples and declaring the targets at the desired confidence level with $O(\log \frac{1}{\epsilon})$ samples.
The proposed strategy applies to unknown heavy-tailed distribution models and preserves its order optimality in both $N$ and $\epsilon$.
Comprising of calculating confidence bounds of the mean
and performing simple comparisons, the proposed strategy is computationally efficient.

\subsection{Related Work}\label{sec:RelatedWork}

The problem studied here is related to several active learning and sequential inference problems. We discuss here
representative studies most pertinent to this paper and emphasize the differences in our approach
from these existing studies.

{\bf{Noisy group testing:}}~Variations of the classic group testing have been extensively studied in the literature focusing mainly on the noisless case. There are several recent studies that consider one-sided error
(false positive or false negative) in the test outcomes~\cite{G3,G4,G5} or symmetric error (with
equal false positive and false negative probabilities)~\cite{G6,G7,G8} with \emph{known} error probabilities. To our best knowledge, the result in this paper is the first applicable to noisy group testing under general and \emph{unknown} noise models.

{\bf{Adaptive sampling with noisy response:}}~The main body of work on adaptive sampling is based on a Bayesian approach with binary noise of a known model.
A popular Bayesian strategyis the Probabilistic Bisection Algorithm (PBA),
which updates the posterior distribution of the step location after each sample (based on the known model of the noisy response) and chooses the next sampling point to be the median point of the posterior distribution. Although several variations of the
method have been extensively studied in the literature~\cite{waeber,BenOr,Castro} following the pioneering
work of~\cite{BZ}, there is little known about the theoretical guarantees,
especially when it comes to unknown noise models. In this paper we present a
non-Bayesian approach to the adaptive sampling problem under general unknown noise models.

{\bf{HH and HHH detection:}}~Prior solutions for online detection of
HHHes typically involve adjusting which prefixes to monitor either at the arrival of each packet~\cite{HHH3,HHH4,HHH5}, or at periodic intervals~\cite{HHH6,HHH7}. A particularly relevant work is the adaptive monitoring algorithm proposed by Jose \etal~\cite{HHH6}, where a fixed number of measurement rules are adjusted at periodic intervals based on the aggregate packet counts matching to each of these rules. At each time interval, the aggregate count is compared to a heuristically chosen threshold (e.g., a fraction of link capacity), to determine whether it is an HHH, and whether the rules need to be kept in the next interval, or expanded
to monitor the children of the prefix, or collapsed and combined with upstream nodes. While the proposed CBRW has a similar flavor of moving among parent and children, which is very much inherent to the HHH detection problem,  the decision criteria used to adjust the prefix is different. Instead of comparing with a fixed threshold, our decision is based on statistical metric determined by the desired detection error. Different from the heuristic studies in the literature, the proposed strategy offers performance grantee and order optimality. We provide a rigorous framework that succinctly captures the tradeoff between detection time and overall detection performance.

{\bf{Pure-exploration bandit problems:}}
The problem studied here is also related to the so-called pure-exploration bandit problems~\cite{ChenLin} where the objective is to search for a subset of bandit arms with certain properties.
In particular, in the Thresholding Bandit problem introduced in~\cite{Locatelli}, the objective is to determine the arms with mean above a given
threshold. The key difference is that the thresholding bandit problem does not assume any structure in the arms and has a linearly growing sample complexity in the number of arms. The focus of this paper is on exploiting the hierarchical structure inherent to many applications to obtain sublinear sample complexity with the size of the search space.

{\bf{Active hypothesis testing for anomaly detection:}}~The active inference problem considered here falls into the
general class of sequential design of experiments pioneered by
Chernoff in 1959 \cite{Act1} with variations and extensions studied in~\cite{Act2,Act3,Act4,Act5}. These studies assume known or parametric models and focus on randomized test strategies. This paper, however, adopts a nonparametric setting and proposes a deterministic strategy. A number of studies on anomaly detection within the sequential and active hypothesis testing framework exist in the literature (see an excellent survey in~\cite{Act6} and a few recent results in~\cite{Act7,Act8,Act9,Act10}). These studies in general assume known models and do not address hierarchical structure of the search space. A particularly relevant work is~\cite{ChaoCohen}, in which a test strategy based on a random walk on a binary tree was developed. While the CBRW policy proposed here shares a similar structure as the test strategy developed in~\cite{ChaoCohen},
the latter work assumes that the stochastic models of all data streams are known and the testing
strategy relies on using the likelihood ratios calculated based on the known distributions. For the
unknown model scenario considered in this paper, the confidence-bound based statistics used for guiding
the biased random walk are fundamentally different from the likelihood ratio. As a result, the
performance analysis also differs.

\section{Problem Formulation and Preliminaries}\label{Sec:ProbFor}

\subsection{Problem Formulation}
Consider a set of $N$ data streams
conforming to a binary-tree structure with $K$ leaf nodes as illustrated in Fig.~\ref{Fig:BT} (extensions to general tree structures are discussed in Sec.~\ref{sec:AL}). Let $(l,k)$ ($l=0,1,\ldots, L, k=1,\ldots, 2^{L-l}$) denote the $k$th node at level $l$ of the tree.
Let $\{X_{k,l}(t)\}_{t=1}^\infty$ denote the corresponding random
process which is independent and identically distributed with an unknown distribution $f_{k,l}$ and an unknown
mean $\mu_{k,l}$.

Associated with each level $l$ of the tree is a given threshold $\eta_l$ that defines the targets of interest. Specifically, a leaf-level target is a node at level $l=0$ whose mean $\mu_{k,0}$ exceeds $\eta_0$.
Hierarchical targets are defined recursively in terms of $l$. Specifically, a hierarchical target at level $l>0$ is a node whose mean value remains above the threshold $\eta_l$ after excluding all its target decedents.
The tree-structured hierarchy encodes the following relationship among nodes in terms of their mean values: for an arbitrary node $(k,l)$, if $\mu_{k,l}>\eta_l$, then $\mu_{k',l'}>\eta_{l'}$ for all $(k',l')$ on the shortest path from $(k,l)$ to the root node of the tree.

An active inference strategy $\pi=(\{a_t\}_{t\ge 1}, T_\pi,\mathcal{S}_\pi)$ consists of a sampling strategy $\{a_t\}_{t\ge 1}$, a stopping rule $\tau_\pi$, and a terminal decision rule $\mathcal{S}_\pi$.
The sampling strategy $\{a_t\}_{t\ge 1}$ is a sequence of functions mapping from past actions and observations to a node of the tree to be sampled at the current time $t$.
The stopping rule $T_\pi$ determines when to terminate the search, and the decision rule $\mathcal{S}_\pi$ declares the detected set of targets at the time of stopping.
Let $\mathbb{E}_{\mathcal{F}}$ and $\mathbb{P}_{\mathcal{F}}$ denote, respectively,
the expectation and the probability measure under distribution
model $\mathcal{F}=\{f_{(k,l)}\}_{\begin{subarray}{1}l=0,1,\ldots, L\\k=1,\ldots, 2^{L-l} \end{subarray}}$.
The objective is as follows:
\begin{eqnarray}\nn
&\hspace{-1.5em}\textrm{minimize}_\pi ~\mathbb{E}_{\mathcal{F}}T_\pi,\\\nn
&\textit{s.t.}~~ \mathbb{P}_{\mathcal{F}}[\mathcal{S}_\pi\neq \mathcal{S}]\le \epsilon,
\end{eqnarray}
where $\mathcal{S}$ is the true set of targets.

\subsection{Sub-Gaussian, Heavy-Tailed, and Concentration Inequality}

We consider a general distribution $f_{k,l}$ for each process. Due to different concentration behaviors, sub-Gaussian and heavy-tailed distributions are treated separately.
Recall that a real-valued random variable $X$ is called sub-Gaussian~\cite{SubG} if, for all $\lambda\in (-\infty,\infty)$,
\begin{eqnarray}\label{zeta}
\mathbb{E}[e^{\lambda(X-\mathbb{E}[X])}]\le e^{\xi \lambda^2/2}
\end{eqnarray}
for some constant $\xi>0$. We assume (an upper bound on) $\xi$ is known.
For sub-Gaussian random variables, Chernoff-Hoeffding concentration
inequalities hold. Specifically~\cite{Chernoff1}:

\begin{eqnarray}\label{Hoeff}
\left\{
\begin{array}{ll}
\mathbb{P}\left[\overline{X}(s)+\sqrt{\frac{2\xi\log\frac{1}{p}}{s}}<\mu\right] \le p\\
\\
\mathbb{P}\left[\overline{X}(s)-\sqrt{\frac{2\xi\log\frac{1}{p}}{s}}>\mu\right] \le p,
\end{array}\right.
\end{eqnarray}
\break
where $\overline{X}(s)=\frac{1}{s}\sum_{t=1}^s X(t)$ is the sample mean of $s$ independent samples of $X$.

For heavy-tailed distributions, the moment generating functions are no longer bounded, and the Chernoff-Hoeffding type of concentration inequalities do not hold.
However, the following can be said about a truncated sample mean statistic in place of the sample mean in~\eqref{Hoeff}. Assume that a $b$-th ($1<b<2$) moment of $X$ is bounded:
\begin{eqnarray}\label{HTBM}
\mathbb{E}[X^b]\le u,
\end{eqnarray}
for some $u>0$. Define the following truncated sample mean with a parameter $p \in (0,\frac{1}{2}]$
\begin{eqnarray}\label{Trunc}
\widehat{X}(s,p)=\frac{1}{s}\sum_{t=1}^{s}X(t)\mathbbm{1}\left\{ |X(t)|\leq(\frac{ut}{\log\frac{1}{p}})^{1/{b}} \right\}.
\end{eqnarray}
We then have (Lemma~1 in~\cite{Bubeck}),
{\begin{eqnarray} \label{HTCon}
\left\{
\begin{array}{ll}
&\Pr\left[\widehat{X}(s,p)-4u^{1/{b}}(\frac{\log\frac{1}{p}}{s})^{\frac{{b}-1}{{b}}}>\mu\right]\leq p\\
\\
&\Pr\left[\widehat{X}(s,p)+4u^{1/{b}}(\frac{\log\frac{1}{p}}{s})^{\frac{{b}-1}{{b}}}<\mu\right]\leq p.
\end{array}\right.
\end{eqnarray}}

\section{An Active Inference Strategy: CBRW}\label{Sec:HTB}

In this section, we present the Confidence Bounds based Random Walk (CBRW) policy.
We focus on the case of a single target and sub-Gaussian distributions. Extensions to multiple target detection and
heavy-tailed distributions are discussed in Sec.~\ref{sec:AL}.

\subsection{Detecting Leaf-Level Targets}
We first consider applications where it is known that the target is at the leaf level. This includes, for example, HH detection, noisy group testing, and adaptive sampling as discussed in Sec.~\ref{appl}.

The basic structure of CBRW consists of a global random-walk module interwoven with a local
CB-based sequential test module at each step of the random walk.
Specifically, the CBRW policy performs a biased random walk on the
tree that eventually arrives and terminates at the target with the
required reliability. Each move in the random walk (i.e., which
neighboring node to visit next) is guided by the output of the local
CB-based sequential test module. This module ensures that the random walk is more likely to move toward
the target than to move away from the target and that the random walk terminates at the true target with high probability.

Consider first the local CB-based sequential test module. This local sequential test is carried out on a specific node (random process) $\{X(t)\}_{t=1}^\infty$, where we have omitted the node index $(k,l)$ for simplicity. The goal is to determine whether the mean value of $\{X(t)\}_{t=1}^\infty$ is below a given
threshold $\eta$ at a confidence level of $1-\beta$ or above the threshold at a confidence level of $1-\alpha$.
If the former is true, the test module outputs $0$, indicating this node is unlikely to be an ancestor of a target or the target itself. If the latter is true, the output is $1$.
Let $\mathcal{L}(\alpha,\beta,\eta)$ denote this local sequential test with given parameters $\{\alpha,\beta, \eta\}$. It sequentially collects samples from $\{X(t)\}_{t=1}^\infty$. After collecting each sample, it determines whether to terminate the test and if yes, which value to output based on the following rule:
%
%
%
\begin{itemize}
\item If $\overline{X}(s)-\sqrt{\frac{2\xi\log\frac{2 s^3}{\alpha}}{s}} >\eta$, terminate and output $1$.
\item If $\overline{X}(s) +\sqrt{\frac{2\xi\log\frac{2 s^3}{\beta}}{s}}<\eta$, terminate and output $0$.
\item Otherwise, continue taking samples,
\end{itemize}
where $\overline{X}(s)$ denotes the sample mean obtained form $s$ observations and $\xi$ is the distribution
parameter specified in~\eqref{zeta}.
\begin{figure}[htbp]
\noindent\fbox{
\parbox{6in}{
\begin{algorithmic}[1]
\small
      \State Initialization: Initial location of the random walk $(k,l)=(1,L)$, $p_0\in(0,1-\frac{1}{\sqrt2})$, $\alpha=\beta=p_0$, $\epsilon\in(0,0.5)$, $STOP=0$.
      \vspace{1em}
\Loop ~while $STOP$=0
\If{$l>1$}
        \State Test the left child of $(k,l)$ by $\mathcal{L}(\alpha,\beta,\eta_{l-1})$
            \If{the output of the test on the left child is $1$}
                 \State Move to the left child of $(k,l)$.
            \ElsIf{the output of the test on the left child is $0$}
                \State Test the right child of $(k,l)$ by $\mathcal{L}(\alpha,\beta,\eta_{l-1})$
                    \If{the output of the test on the right child is $1$}
                         \State Move to the right child of $(k,l)$.
                     \ElsIf{the output of the test on the right child is $0$}
                         \State Move to the parent of $(k,l)$.
                    \EndIf
            \EndIf
\ElsIf{$l=1$}
    \State Test the left child of $(k,l)$ by $\mathcal{L}(\frac{\epsilon}{2L C_{p_0}},\beta,\eta_{l-1})$.
         \If{the output of the test on the left child is $1$}
            \State Declare the left child of $(k,l)$as the target.
            \State Set STOP=1.
        \ElsIf{the output of the test on the left child is $0$}
            \State Test the right child of $(k,l)$ by $\mathcal{L}(\frac{\epsilon}{2L C_{p_0}},\beta,\eta_{l-1})$.
                    \If{the output of the test on the right child is $1$}
                        \State Declare the right child of $(k,l)$as the target.
                        \State Set STOP=1.
                    \ElsIf{the output of the test on the right child is $0$}
                          \State Move to the parent of $(k,l)$.
                    \EndIf
        \EndIf
 \EndIf
\EndLoop

\end{algorithmic}
}}
\caption{The random walk module of CBRW for detecting a single leaf-level target.}\label{Fig:RWCB}
\end{figure}
We now specify the random walk on the tree based on the outputs of the local CB-based tests.
Let $(k,l)$ denote the
current location of the random walk (which is initially set at the root node).
Consider first $l>1$. The left
child of $(k,l)$ is first probed by the local module $\mathcal{L}(\alpha,\beta,\eta)$ with parameters set to $\alpha=\beta=p_0$
where $p_0$ can be set to any constant in $(0,1-\frac{1}{\sqrt 2})$, and $\eta$ being the threshold associated with level $l-1$ of the children of $(k,l)$.
If the output is $1$, the random walk moves to the left child of $(k,l)$, and the procedure repeats.  Otherwise, the right
child of $(k,l)$ is tested
with the same set of parameters, and the random walk moves to the right child if this test outputs $1$.
If the outputs of the tests on both children are $0$, the
random walk moves back to the parent of $(k,l)$ (the parent of the
root node is defined as itself). The values for $\{\alpha,\beta,\eta\}$ specified above ensure that the random walk moves toward the target with a probability greater than $1/2$.
When the random walk arrives at a node on level $l=1$,
the left
child of $(k,l)$ is first probed by the local module $\mathcal{L}(\alpha,\beta,\eta)$ with parameters set to $\alpha=\frac{\epsilon}{2L C_{p_0}}$, $\beta=p_0$ and $\eta=\eta_0$ where
\begin{eqnarray}\label{CP0}
C_{p_0}=\frac{1}{\bigg(1-\exp(-2(1-2(1-p_0)^2)^2)\bigg)^2}.
\end{eqnarray}
If the output is $1$, the random walk terminates and the left child of $(k,l)$ is declared as the target.  Otherwise, the right
child of $(k,l)$ is tested
with the same set of parameters, and the random walk terminates with the right child declared as the target if this test outputs $1$.
If the outputs of the tests on both children are $0$,
random walk moves back to the parent of $(k,l)$.
The values for $\{\alpha,\beta,\eta\}$ specified above ensure that the random walk terminates at and declares the true target with the required confidence level of $1-\epsilon$. A description of CBRW for detecting a single leaf-level target is given in Fig.~\ref{Fig:RWCB}.

\begin{figure}[htbp]
\noindent\fbox{
\parbox{6in}{
\begin{algorithmic}[1]
\small
      \State Initialization: Initial location of the random walk $(k,l)=(1,L)$, parameters $STOP=0$, $p_0\in(0,1-\frac{1}{\sqrt[3]2})$, $\alpha=\beta=p_0$, $\epsilon\in(0,0.5))$, $C^H_{p_0}$, $STOP=0$.
\Loop~while $STOP=0$
\If{$l>0$}
            \State Test $(k,l)$ by $\mathcal{L}(\alpha,\beta,\eta_{l})$
                \If{the output of the test on $(k,l)$ is $0$}
                    \State Move the parent of $(k,l)$. Set $\alpha=\beta=p_0$.
                \ElsIf{the output of the test on $(k,l)$ is $1$}
                    \State Test the left child of $(k,l)$ by $\mathcal{L}(\alpha,\beta,\eta_{l-1})$.
                        \If{the output of the test on the left child of $(k,l)$ is $1$}
                            \State Move to the left child of $(k,l)$. Set $\alpha=\beta=p_0$.
                        \ElsIf{the output of the test on the left child of $(k,l)$ is $0$}
                            \State Test the right child of $(k,l)$ by $\mathcal{L}(\alpha,\beta,\eta_{l-1})$.
                                \If{the output of the test on the right child of $(k,l)$ is $1$}
                                \State Move to the right child of $(k,l)$. Set $\alpha=\beta=p_0$.
                                \ElsIf{the output of the test on the right child of $(k,l)$ is $0$}
                                    \If{$\alpha<\frac{\epsilon}{3L C^H_{p_0}}$}
                                        \State Declare $(k,l)$ as the target.
                                        \State Set $STOP=1$.
                                     \Else
                                        \State Divide $\alpha$ and $\beta$ by $2$: $\alpha\leftarrow\frac{\alpha}{2}$ and $\beta\leftarrow\frac{\beta}{2}$.
                                     \EndIf
                                \EndIf

                        \EndIf
                \EndIf

\ElsIf{$l=0$}
            \State Test $(k,l)$ by $\mathcal{L}(\frac{\epsilon}{3L C^H_{p_0}},\beta,\eta_{l})$
                \If{the output of the test on $(k,l)$ is $0$}
                    \State Move to the parent of $(k,l)$.
                \ElsIf{the output of the test on $(k,l)$ is $1$}
                    \State Declare $(k,l)$ as the target.
                    \State Set $STOP=1$.
                \EndIf
\EndIf
\EndLoop

\end{algorithmic}
}}
\caption{The random walk module of CBRW for detecting a single hierarchical target.}\label{Fig:RWCB2}
\end{figure}

\subsection{Detecting Hierarchical Targets}

We now consider the case where the target may reside at higher levels of the tree. The following modified CBRW policy detects a potentially hierarchical target with the required confidence level of $1-\epsilon$.

Let $(k,l)$ denote the
current location of the random walk (which is initially set at the root node).
Consider first $(k,l)$ is a non-leaf node with $l>0$. The node $(k,l)$ is first probed by the local module $\mathcal{L}(\alpha,\beta,\eta)$ with parameters set to $\alpha=\beta=p_0$ where $ p_0\in(0,1-\frac{1}{\sqrt[3]2})$
and $\eta=\eta_l$. If the output is $0$, the random walk moves to the parent of $(k,l)$. If the output
is $1$,
then the left child of $(k,l)$ is tested by the local module $\mathcal{L}(\alpha,\beta,\eta)$ with parameters set to $\alpha=\beta=p_0$ and $\eta=\eta_{l-1}$. If the output is $1$, the random walk moves to the left child. Otherwise, the right
child of $(k,l)$ is tested
with the same set of parameters, and the random walk moves to the right child if this test outputs $1$.
If the outputs of the tests at $(k,l)$ and its children are $1$, $0$, and $0$, respectively,
then $(k,l)$ is likely to be a hierarchical target and the random walk stays at $(k,l)$.
When the random walk stays at the same node $(k,l)$, the same tests are repeated on $(k,l)$ and
its children with an increased confidence level. We increase the confidence level by dividing $\alpha$ and $\beta$ by $2$ iteratively.
%
When the current value of $\alpha$ and $\beta$ becomes smaller than $\frac{\epsilon}{3L C^H_{p_0}}$, the random walk stops and declares $(k,l)$ as the target.
The value of
\begin{eqnarray}\label{eq91}
C^H_{p_0}=\frac{1}{\bigg(1-\exp(-2(1-2(1-p_0)^3)^2)\bigg)^2}
\end{eqnarray}
ensures the
desired confidence level of $1-\epsilon$ at detection of the target. If the random walk moves to a new location the values of $\alpha$ and $\beta$ is reset to $p_0$. When the random walk arrives at a leaf node $(k,l)$ with $l=0$, the leaf node is tested by the local module $\mathcal{L}(\alpha,\beta,\eta)$ with parameters set to $\alpha=\frac{\epsilon}{3L C^H_{p_0}}$, $\beta=p_0$ and $\eta=\eta_0$. If the output is $1$, the random walk stops and declares $(k,l)$ as the target. Otherwise, the random walk moves to the parent of $(k,l)$. A description of CBRW for detecting a single hierarchical target is given in Fig.~\ref{Fig:RWCB2}.

\section{Performance Analysis}\label{Sec:PAnal}
In this section, we analyze  the sample complexity of CBRW.
We again focus
on the case of a single target and sub-Gaussian distributions and leave extensions to more general cases to Sec.~\ref{sec:AL}.

\subsection{The Sample Complexity of the CB-based Sequential Test Module}

To analyze the sample complexity of CBRW, we first analyze the sample complexity of the local CB-based sequential test module $\mathcal{L}(\alpha, \beta, \eta)$ in the lemma below. We then analyze the
behavior of the random walk to establish the number of times that the local sequential test is carried out.


\begin{lemma}\label{SCofP}
Let $\mu$ denote the expected value of an i.i.d. sub-Gaussian random process $\{X(t)\}_{t=1}^\infty$.
Let $\tau_{\mathcal{L}}$ be the stopping time of the
CB-based sequential test $\mathcal{L}(\alpha, \beta, \eta)$ applied to $\{X(t)\}_{t=1}^\infty$.
We have, in the case of $\mu>\eta$,
\begin{eqnarray}\label{ProbL01}
&&\hspace{-3em}\mathbb{P}[\overline{X}({T})+\sqrt{\frac{2\xi\log\frac{2 T^3}{\beta}}{T}}<\eta]\le \beta,\\\label{ProbL1}
&&\hspace{-3em}\mathbb{E}[T]\le\frac{48}{(\mu-\eta)^2}\log \frac{24\sqrt[3]{\frac{2}{\alpha}}}{(\mu-\eta)^2}+2.
\end{eqnarray}
In the case of $\mu<\eta$,

\begin{eqnarray}\label{ProbL02}
&&\hspace{-3em}\mathbb{P}[\overline{X}({T})-\sqrt{\frac{2\xi\log\frac{2 T^3}{\alpha}}{2T}}>\eta]\le \alpha,\\\label{ProbL2}
&&\hspace{-3em}\mathbb{E}[T]\le\frac{48}{(\mu-\eta)^2}\log \frac{24\sqrt[3]{\frac{2}{\beta}}}{(\mu-\eta)^2}+2.
\end{eqnarray}

\end{lemma}
\begin{proof}
See Appendix~A.
\end{proof}

The inequalities \eqref{ProbL01} and \eqref{ProbL02} establish the confidence levels for the local sequential CB-based test.
Both results on the confidence levels and the sample complexity are based on the concentration inequalities given in~\eqref{Hoeff}.

\subsection{The Sample Complexity of CBRW}

In both cases of leaf-level target detection and hierarchical target detection, the sample complexity of CBRW is order optimal in both $N$ and $\frac{1}{\epsilon}$ as stated in Theorem~\ref{SCofRWCB} below.
\begin{theorem}\label{SCofRWCB}
Assume that there exists $\delta>0$ such that
$\mu_{k,l}-\eta_l\ge \delta$ for all $(k,l)$ ($l=0,1,\ldots, L, k=1,\ldots, 2^{L-l}$). We have\footnote{The search
space for the case of detecting leaf-level target is of size $K$, the number of leaf nodes. However, since $N=2K-1$ is of the same order of $K$,~\eqref{logord} satisfies for both $K$ and $N$.}
\begin{eqnarray}\label{logord}
\mathbb{E}_\mathcal{F}[T_{CBRW}]= O(\log_2 N +\log\frac{1}{\epsilon})
\end{eqnarray}
and
\begin{eqnarray}
\mathbb{P}_\mathcal{F}[\mathcal{S}_{CBRW}\neq \mathcal{S}]\le \epsilon.
\end{eqnarray}

\end{theorem}
\begin{proof}
See Appendix B.
\end{proof}

The gap $\mu_{k,l}-\eta_l$ in the mean value of a random process at level $l$ and the threshold at the respective level indicates the informativeness of the observations.
It is a practical assumption that the gap is bounded away from $0$.
We might also naturally assume that the higher levels are less informative; thus, have smaller gaps. For example, in group testing, tests from larger groups of items
are less informative about the presence of defective items. Under this assumption we can provide a finite-time upper bound on the sample complexity of CBRW.

We first introduce some auxiliary notions which are useful in understanding the trajectory of the random walk in CBRW. Under leaf-level target setting, consider a sequence of
subtrees $\{\mathcal{T}_1,\mathcal{T}_2,...,\mathcal{T}_L\}$ of $\mathcal{T}$. Subtree
$\mathcal{T}_L$ is
obtained by removing the biggest half-tree containing the target from $\mathcal{T}$.
Subtree $\mathcal{T}_l$ is iteratively obtained by removing the biggest half-tree containing the target from the half-tree containing the target in the previous iteration.
In the example given in Fig.~\ref{Fig:BT}, $\mathcal{T}_3=\{(1,3),(2,2),(3,1),(4,1),(5,0),(6,0),(7,0),(8,0)\}$, $\mathcal{T}_2=\{(1,2),(2,1),(3,0),(4,0)\}$ and $\mathcal{T}_1=\{(1,1),(2,0)\}$.
Let $(r_l,l)$ denote the child of the root node of $\mathcal{T}_l$.
Under hierarchical target setting with a hierarchical target at level $l_0$, let subtrees $\{\mathcal{T}_{l_0+1},..., \mathcal{T}_L\}$ be the same as defined for a target that is a
decedent of the hierarchical target. Also, define $\mathcal{T}'_1$ and $\mathcal{T}'_2$ as subtrees whose root nodes are
children of the hierarchical target. Let $(r'_l,l)$ denote the root node of $\mathcal{T}'_l$.

A detailed finite-time upper bound on the sample complexity of CBRW is given in~\eqref{eq9} and~\eqref{eq91} where $(r_0,0)$ and $(r_{l_0},l_0)$ denote the targets under leaf-level and hierarchical settings, respectively.

\begin{figure*}
\begin{center}
\noindent\fbox{
\parbox{6.5in}
{\small{
\begin{eqnarray}\label{eq9}
\mathbb{E}_\mathcal{F}[T_{CBRW}]
&\le& 2\sum_{l=1}^{L} C_{p_0} \bigg(\frac{48}{(\mu_{r_l,l}-\eta_l)^2}\log \frac{24\sqrt[3]{\frac{2}{p_0}}}{(\mu_{r_l,l}-\eta_l)^2}+2\bigg)
+ \frac{48}{(\mu_{r_0,0}-\eta_0)^2}\log \frac{24\sqrt[3]{\frac{2C_{p_0}L}{\epsilon}}}{(\mu_{r_0,0}-\eta_0)^2}+2.~~~~~\\\nn
\mathbb{E}_\mathcal{F}[T_{CBRW}]
&\le& 3\sum_{l=1_0+1}^{L} \frac{C^H_{p_0}}{1-p_0} \bigg[\bigg(\frac{48}{(\mu_{r_l,l}-\eta_l)^2}\log \frac{24\sqrt[3]{\frac{2}{p_0}}}{(\mu_{r_l,l}-\eta_l)^2}+2\bigg)
+\frac{p_0}{(1- p_0)}\frac{16\log 2}{(\mu_{r_l,l}-\eta_l)^2}\bigg]\\\nn
&&~~~+3\sum_{l'=1}^{2} \frac{C^H_{p_0}}{1-p_0} \bigg[\bigg(\frac{48}{(\mu_{r_l',l'}-\eta_{l'})^2}\log \frac{24\sqrt[3]{\frac{2}{p_0}}}{(\mu_{r_l',l'}-\eta_{l'})^2}+2\bigg)
+\frac{p_0}{(1- p_0)}\frac{16\log 2}{(\mu_{r_{l'},l'}-\eta_{l'})^2}\bigg]\\\label{eq91}
&&~~~+ {\log_2\frac{6LC^H_{p_0}p_0}{\epsilon}}\bigg(\frac{48}{(\mu_{r_{l_0},l_0}-\eta_{l_0})^2}\log \frac{24\sqrt[3]{\frac{4}{\epsilon}}}{(\mu_{r_{l_0},l_0}-\eta_{l_0})^2}+2\bigg).
\end{eqnarray}}

}}
\end{center}\caption{Finite-regime upper bounds on the performance of CBRW under leaf-level~\eqref{eq9} and hierarchical~\eqref{eq91} target settings.}
\end{figure*}


\section{Extensions and Discussions}\label{sec:AL}

In this section, we discuss extensions to more general scenarios and the mapping of various applications to the target detection problem at hand.

\subsection{Detecting an Unknown Number of Targets}

Detecting $|\mathcal{S}|>1$ targets with $|\mathcal{S}|$ known can be easily implemented by
sequentially locating the targets one by one. We assume that each target
can be removed after it is located by CBRW\footnote{For example, in group testing, the detected defective item is no longer tested in any subsequent group tests or in HHH detection, the packet count of each detected HHH can be subtracted from the packet count of the parents.}. To ensure that the reliability constraint holds, we replace $\epsilon$ with $\frac{\epsilon}{|\mathcal{S}|}$ in each search of a single target. The reliability constraint holds by union bound on the error probabilities of the searches for a single target.

When the number of targets is unknown, but an upper bound $S_{\textrm{max}}\ge |\mathcal{S}|$ on the number of targets is known, we can similarly detect the targets one by one.
To ensure that the reliability constraint holds, we replace $\epsilon$ with $\frac{\epsilon}{2S_{\textrm{max}}}$ in each search of a single target.
The stopping rule for the overall search can be implemented by testing the root node. Specifically, the root node is tested by $\mathcal{L}(\epsilon_0,\epsilon_0,\eta_L)$ with $\epsilon_0=\frac{\epsilon}{2S_{\textrm{max}}}$ every $LC_{p_0}$ steps in the random walk under leaf target setting and every $LC^H_{p_0}$ steps under the hierarchical target setting. The reliability constraint holds by union bound on the error probabilities of the searches for a single target and error in stopping the overall search before finding all targets.

Under both leaf-level and hierarchical
target settings, with this modification, the sample complexity of finding single targets simply add up to ab $O(|\mathcal{S}|\log K + |\mathcal{S}|\log \frac{1}{\epsilon})$ overall sample complexity.

\subsection{Heavy-Tailed Distributions}

The extension to more general distribution models can be implemented
by only modifying the local CB-based test $\mathcal{L}$ in a way that the confidence levels remain the same. As a result, the
behavior of the random walk on the tree remains the same.

Specifically, for heavy-tailed distributions with existing $b$'th moment as given in~\eqref{HTBM}, we modify the test $\mathcal{L}$
\begin{itemize}
\item If $\widehat{X}(s,\alpha)-4u^{1/{b}}(\frac{\log\frac{2s^3}{\alpha}}{s})^{\frac{{b}-1}{{b}}}>\eta$, terminate and output $1$.
\item If $\widehat{X}(s,\beta)+4u^{1/{b}}(\frac{\log\frac{2s^3}{\beta}}{s})^{\frac{{b}-1}{{b}}}<\eta$, terminate and output $0$.
\item Otherwise, continue taking samples.
\end{itemize}
The resulting CBRW achieves the same $O(\log N+ \log\frac{1}{\epsilon})$ sample complexity under both leaf-level and hierarchical target settings. The proofs follow similar to the proofs of Theorem~\ref{SCofRWCB} and Lemma~\ref{SCofP}, using confidence bounds~\eqref{HTCon} instead of~\eqref{Hoeff} in the proof of Lemma~\ref{SCofP}.

\subsection{General Tree Structures}

Consider a general tree-structured hierarchy as
shown in Fig.~\ref{Fig:GenT}. The CBRW policy can be modified as follows.

To have the required confidence level in taking the steps toward the target, the input parameters in the
local CB-based sequential test $\mathcal{L}$ are modified based on the degree $d_{k,l}$
of each node $(k,l)$ in the tree. In particular, under the leaf-level target setting, we choose $p_0\in(1-\frac{1}{2^{-(d_{k,l}-1)}})$
and $\alpha=\frac{\epsilon}{(D-1)LC}$ where $L$ is the
maximum distance from the root node to a leaf node, $D$ is the maximum
degree of the nodes in the tree and $C$ is a constant independent of $K$ and $\epsilon$.
Under the hierarchical target setting,
we choose $p_0\in(1-\frac{1}{2^{-d_{k,l}}})$ and when increasing the confidence level iteratively to detect the hierarchical target, we terminate the search when $p$ goes below $\frac{\epsilon}{DLC}$.
The random walk moves
to a child or the parent of the current location according
to the outputs of the tests.

Following similar lines as in
the proof of Theorem~\ref{SCofRWCB}, we can show a
sample complexity of $O(LD)+O(\log\frac{1}{\epsilon})$ under both leaf-level and hierarchical target settings.

\begin{figure}[htbp]
\centering
\begin{psfrags}
\scalefig{0.6}\epsfbox{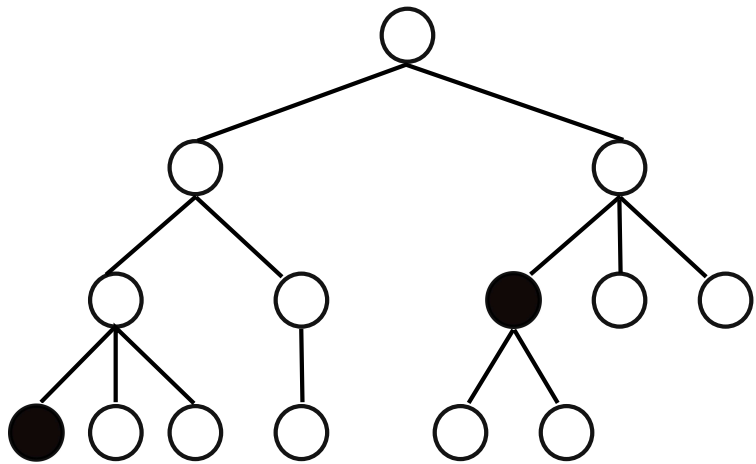}\caption{A general tree-structured hierarchy.} \label{Fig:GenT}
\end{psfrags}
\end{figure}

\subsection{Mapping from Various Applications}

{\bf{HH and HHH detection:}}~The CBRW policy directly applies to leaf-level HH and HHH detection under leaf-level and hierarchical target settings, respectively.
In particular, provided a controllable counter which can be assigned to each IP prefix, we assign the counter to the current location which is
desired to be tested according to $\mathcal{L}$. Based on the packet count, the counter is moved on the tree according to CBRW.
When there are several counters available the tree can be partitioned to smaller subtrees and CBRW is run on each subtree separately to make
an efficient use of the available counters.

{\bf{Noisy group testing:}}~
In group testing, the objective is to identify a few defective items in a population of $K$ items performing tests on subsets of items. Each group test gives a binary outcome, indicating
whether the tested group contains any defective items. We consider the case of adaptive group testing under general and unknown noise models for test outcomes. Specifically,
the outcome of a group test is a Bernoulli random variable with an unknown parameter that represents the false alarm or missed detection probability. 

The CBRW policy under leaf-level target setting directly applies to noisy group testing where
the defective items are the leaf-level targets. Each data stream corresponds to the noisy Bernoulli test outcomes of the given subsets of the population.
The parent-children relationship on the tree represents the subset relationship among the corresponding test groups such that the group corresponding to the
parent is the union of the groups corresponding to the children. Under the practical assumption that false alarm
and miss detection probabilities are smaller than $1/2$, targets are
those data streams generated by a singleton subset (i.e., leaf nodes) with
a mean value exceeding $1/2$. Although group testing problem does not necessarily conform to a predetermined hierarchical structure, the proposed solution offers order optimal number of tests in both
population size and reliability constraint.

{\bf{Adaptive sampling with noisy response:}}~
Consider
the $[0,1]$ interval as the input space. We limit the input space to be one-dimensional in order
to demonstrate the main idea. The hypothesis
class, denoted by $\mathcal{H}$, is the set of all step functions on [0,1] interval.

\begin{eqnarray}~\label{AdSa}
\mathcal{H}=\bigg\{h_z: [0,1]\rightarrow \mathbb{R}, h_z(x)=\mathbbm{1}_{\{(z,1]\}}(x), z\in(0,1)\bigg\}
\end{eqnarray}
Each hypothesis $h_z$ assigns a binary label to each element of the
input space $[0,1]$. There is a true hypothesis $h_{z^*}$ that
determines the ground truth labels for the input space.

The learner is allowed to make sequential observations by
adaptively sampling $h_{z^*}$. The observations are however
noisy.
The goal is to design a sequential sampling strategy aiming
at minimizing the sample complexity
required to obtain a confidence interval of length $\Delta$
for $z^*$ at a $1-\epsilon$ confidence level.
Specifically, the learner chooses the sampling point $x$
at each time $t$ and receives a noisy sample of the true
hypothesis.

We consider two noise models with unknown distribution.
In the first noise model, the learner observes a noisy sample of the threshold function in the form of
\begin{eqnarray}h^N_{z^*}(x;t)=h_{z^*}(x;t)+n(x,t),\end{eqnarray}
where $n(x,t)$ is a zero mean sampling noise that possibly depends on
the sampling point~$x$ and is generated i.i.d. over~$t$.

In the second noise model, the binary samples can flip
from zeros to ones and vice versa. Specifically, the learner receives
erroneous binary samples with an error probability of $p(.)$ in the form of
\begin{eqnarray}h^B_{z^*}(x;t)=h_{z^*}(x;t)\oplus B(x,t),\end{eqnarray}
where $\oplus$ is the boolean sum and $B(x,t)$ is a Bernoulli
random variable with~$\mathbb{P}[B(x,t)=1]=p(x)$ that may depend on the sampling point~$x$ and is generated i.i.d. over~$t$.

We now present a solution to the adaptive sampling problem based
on the results obtained for CBRW strategy.
For the simplicity of
presentation we assume $\Delta=\frac{1}{2^L}$ ($K=\frac{1}{\Delta}$).
Let each node on a binary tree $\mathcal{T}$ represent an
interval $[z^{L}_{k,l}, z^{R}_{k,l}]\subset[0,1]$ with $z^{L}_{k,l}=(k2^{L-l}-1)\Delta$ and $z^{U}_{k,l}=k2^{L-l}\Delta$. The interval corresponding to each node on the tree is the union of the
intervals corresponding to its children.

What remains to be specified is what entails when probing a node/interval $(k,l)$.
When $(k,l)$ is probed the boundary points of the interval are tested by $\mathcal{L}(\alpha,\beta,\eta)$ with parameters set to $\eta_l=0.5$ and
$\alpha=\beta=p_0$ on a non-leaf node, and $\alpha=\frac{\epsilon}{2LC_{p_0}}$, $\beta=p_0$ on a leaf node where $p_0$ can be set to any constant
in $(0,1-\frac{1}{\sqrt[4]2})$. 
%
The output is $1$ (indicating that the interval is likely to contain $z^*$) if and only if the output of
$\mathcal{L}$ is $0$ on the left boundary and $1$ on the right boundary.
The output is $0$ otherwise.
From the results on the analysis of CBRW the above solution has a sample complexity of $O(\frac{1}{c^2}\log K + \frac{1}{c^2}\log \frac{1}{\epsilon})$ where
$c$ is $0.5$ under the first noise model and $c$ is a lower bound on the gap in $0.5-p(.)$ under the second noise model.

%
%
%
%
%
%
%

\section{Conclusion}

In this paper, we studied the problem of detecting a few targets among a
large number of hierarchical data streams modeled as random processes with unknown distributions.
We designed a sequential strategy to interactively choose the sampling point aiming at minimizing the sample
complexity subject to a reliability constraint. The proposed sequential sampling strategy detects the targets at the desired confidence level with an order optimal logarithmic
sample complexity in both problem size and the parameter of reliability constraint. We further showed the
extensions of the results to a number of active inference
and learning problems in networking and data analytics
applications.

The results obtained in this work extend to the detection of anomaly in other statistics such as variance. In particular,
replacing the sample mean with sample variance in the local CB-based sequential test and modifying the second term in
the upper and lower confidence bounds in the local CB-based sequential test,
both the CBRW policy and its sample complexity analysis apply
to anomaly detection
where the anomalies manifest in the variance. Similar results can be obtained for other statistics assuming the existence of an efficient estimator.

\section*{Appendix~A}

\begin{proof}[Proof of Lemma~\ref{SCofP}]

The proof of Lemma~\ref{SCofP} is based on concentration inequalities for Sub-Gaussian distributions.

We prove inequality~\eqref{ProbL1} here. The other case, $\mu<\eta$ ~\eqref{ProbL2}, can be proven similarly.
\begin{eqnarray}\nn
&&\hspace{-4em}\mathbb{P}\bigg[\overline{X}{(T)}+\sqrt{\frac{2\xi\log\frac{2 T^3}{\beta}}{T}}<\eta\bigg]\\\nn
&\le&\mathbb{P}\bigg[\sup_s\overline{X}{(s)}+\sqrt{\frac{2\xi\log\frac{2 s^3}{\beta}}{s}}<\eta\bigg]\\\nn
&\le&\sum_{s=1}^\infty \mathbb{P}\bigg[\overline{X}{(s)}+\sqrt{\frac{2\xi\log\frac{2 s^3}{\beta}}{s}}<\eta\bigg]\\\label{CI3}
&\le&\sum_{s=1}^\infty\exp(-\log\frac{2 s^3}{\beta} )\\\nn
&=&\sum_{s=1}^\infty\frac{p}{2 s^3}\\\nn
&\le&\beta.
\end{eqnarray}
Inequity~\eqref{CI3} is obtained by~\eqref{Hoeff}.

We next analyze the $\mathbb{E}[T]$ for $\mu>\eta$. Let $s_0=\min\{s\in \mathbb{N}: \sqrt{\frac{2\log\frac{2 s^3}{\alpha}}{s}}\le\frac{\mu-\eta}{2},s>1\}$, for $n\ge s_0$:
\begin{eqnarray}\nn
\mathbb{P}[T\ge n]
&\le& \mathbb{P}\bigg[\sup\bigg\{s:\overline{X}{(s)}+\sqrt{\frac{2\xi\log\frac{2 s^3}{\beta}}{s}}>\eta,~\text{and}~\\\nn
&&\hspace{1em}\overline{X}{(s)}-\sqrt{\frac{2\xi\log\frac{2 s^3}{\alpha}}{s}}<\eta\bigg\}\ge n\bigg]\\\nn
&\le& \mathbb{P}\bigg[\sup\bigg\{s:~\overline{X}{(s)}-\sqrt{\frac{2\xi\log\frac{2 s^3}{\alpha}}{s}}<\eta\bigg\}\ge n\bigg]\\\nn
&\le&\sum_{s=n}^\infty\mathbb{P}\bigg[\overline{X}{(s)}-\sqrt{\frac{2\xi\log\frac{2 s^3}{\alpha}}{s}}<\eta\bigg]\\\label{XSS}
&\le&\sum_{s=n}^\infty\mathbb{P}\bigg[\overline{X}_{s}-\mu<-\sqrt{\frac{2\log\frac{2 s^3}{\alpha}}{s}}\bigg]\\\nn
&\le&\sum_{s=n}^\infty\exp(-\log\frac{2 s^3}{\alpha} )\\\nn
&\le&\sum_{s=n}^\infty\frac{\alpha}{2s^3}\\\nn
&\le&\frac{\alpha}{4(n-1)^2}.
\end{eqnarray}
Notice that~\eqref{XSS} holds because $n\ge s_0$. We can write $\mathbb{E}[T]$ in terms of $\mathbb{P}[T\ge n]$ as
\begin{eqnarray}
\mathbb{E}[T]&=&\sum_{n= 0}^\infty\mathbb{P}[T\ge n]\\\nn
&=&s_0+\sum_{n= s_0}^\infty\mathbb{P}[T\ge n]\\\nn
&\le&s_0+\sum_{n= s_0}^\infty\frac{\alpha}{4(n-1)^2}\\\nn
&\le&s_0+1.
\end{eqnarray}
For the last inequality notice that $s_0$ is defined to be bigger than $1$. It remains to find $s_0$. Note that for all $x>0$ we have $\log x <\sqrt x$ so $\log \log x < \log \sqrt x=\frac{1}{2}\log x$. For $s=\frac{48}{(\mu-\eta)^2}\log \frac{24\sqrt[3]{\frac{2}{\alpha}}}{(\mu-\eta)^2} $

\begin{eqnarray}\nn
\log \frac{2s^3}{\alpha}&=& 3\log \sqrt[3]{\frac{2}{\alpha}}s\\\nn
&=&3\log \sqrt[3]{\frac{2}{\alpha}}\frac{48}{(\mu-\eta)^2}\log \frac{24\sqrt[3]{\frac{2}{\alpha}}}{(\mu-\eta)^2} \\\nn
&=&3\log \sqrt[3]{\frac{2}{\alpha}}\frac{24}{(\mu-\eta)^2}+3\log\log \bigg(\frac{24\sqrt[3]{\frac{2}{\alpha}}}{(\mu-\eta)^2}\bigg)^2\\\nn
&\le&3\log \sqrt[3]{\frac{2}{\alpha}}\frac{24}{(\mu-\eta)^2}+3\log \sqrt[3]{\frac{2}{\alpha}}\frac{24}{(\mu-\eta)^2}\\\nn
&=&6\log \sqrt[3]{\frac{2}{\alpha}}\frac{24}{(\mu-\eta)^2}\\\nn
&=&6\frac{(\mu-\eta)^2}{48}s.
\end{eqnarray}
Thus, for $s=\frac{48}{(\mu-\eta)^2}\log \frac{24\sqrt[3]{\frac{2}{\alpha}}}{(\mu-\eta)^2} $,
\begin{eqnarray}\nn
\sqrt{\frac{2\log\frac{2 s^3}{\alpha}}{s}}\le\frac{\mu-\eta}{2}.
\end{eqnarray}
So, we have the following upper bound for $s_0$
\begin{eqnarray}
s_0\le \lceil \frac{48}{(\mu-\eta)^2}\log \frac{24\sqrt[3]{\frac{2}{\alpha}}}{(\mu-\eta)^2} \rceil +1.
\end{eqnarray}
The addition of $1$ is because $s_0$ is defined to be bigger than~$1$.
Thus,

\begin{eqnarray}
\mathbb{E}[T]\le\frac{48}{(\mu-\eta)^2}\log \frac{24\sqrt[3]{\frac{2}{\alpha}}}{(\mu-\eta)^2}+2,
\end{eqnarray}
which completes the proof.

\end{proof}

\section*{Appendix~B}

\begin{proof}[Proof of Theorem~\ref{SCofRWCB}]

An upper bound on the sample complexity of test $\mathcal{L}$ is provided in
Lemma~\ref{SCofP}. Here, we establish an upper bound on the number of times that
test $\mathcal{L}$ is called in CBRW. First, consider the case of leaf-level target setting.
%

In order to analyze the trajectory of the random walk, we consider
the last passage time $T_l$ of the random walk from each subtree $\mathcal{T}_l$.
We prove an upper bound on $\mathbb{E}[T_l]$ for each $l$ which gives an
upper bound on the total number
of times that test $\mathcal{L}$ is called. Notice that the total number
of times that test $\mathcal{L}$ is called is not bigger than $2\sum_{l=1}^L \mathbb{E}[T_l]$.

%

The random walk initially starts at the root node at distance $L$ from the target.
Define the parameters $W_t$ as the steps of the random walk: $W_t=1$ if the random walk moves one step further from the target at time $t$, $W_t=-1$ if
the random walk moves one step closer to the target, and $W_t=0$ when the random walk does not move. Clearly, $\sum_{t=1}^\tau W_t=-L$ where $\tau$ is the
stopping time of the random walk. The random walk stops when the policy declares a leaf node as the target. For the mean value of $W_t$, from Lemma~\ref{SCofP}, we have
%
\begin{eqnarray}\nn
\mathbb{E}[W_t]&=&\mathbb{P}[W_t=1]-\mathbb{P}[W_t=-1]\\\nn
&\le& 1-2(1-p_0)^2\\\nn
&<&0.
\end{eqnarray}
Notice that if the random walk is within the subtree $\mathcal{T}_L$ at step $t$, we have
\begin{eqnarray}
\sum_{s=1}^tW_s>0.
\end{eqnarray}
Thus, we can write
\begin{eqnarray}
\mathbb{P}[T_{L}>n]&\le& \mathbb{P}\bigg[\sup\{t\ge1: \sum_{s=1}^t W_s>0\}>n\bigg]\\\nn
&\le&\sum_{t=n}^\infty \mathbb{P}\bigg[\sum_{s=1}^t W_s>0\bigg]\\\label{CI1}
&\le&\sum_{t=n}^\infty \exp\bigg(-\frac{1}{2}t(1-2(1-p_0)^2)^2\bigg)\\\nn
&=&\frac{\exp(-2n(1-2(1-p_0)^2)^2)}{1-\exp(-2(1-2(1-p_0)^2)^2)}.
\end{eqnarray}
Inequity~\eqref{CI1} is obtained by Hoeffding inequality for Bernoulli distributions.
We can obtain $\mathbb{E}[T_{L}]$ from $\mathbb{P}[T_{L}>n]$ based on the sum of tail probabilities as
\begin{eqnarray}\nn
\mathbb{E}[T_L]&=&\sum_{n=0}^\infty\mathbb{P}[T_L>n]\\\nn
&\le&\sum_{n=0}^\infty\frac{\exp(-2n(1-2(1-p_0)^2)^2)}{1-\exp(-2(1-2(1-p_0)^2)^2)}\\\nn
&=&\frac{1}{\bigg(1-\exp(-2(1-2(1-p_0)^2)^2)\bigg)^2}.
\end{eqnarray}

Let us define
\begin{eqnarray}\label{CP02}
C_{p_0}=\frac{1}{\bigg(1-\exp(-2(1-2(1-p_0)^2)^2)\bigg)^2},
\end{eqnarray}
which is a constant independent of $K$ and $\epsilon$.
From the symmetry of binary tree, it can be seen that $\mathbb{E}[T_l]\le C_{p_0}$ for all $l$ and the expected number of points visited by the random walk is upper bounded by $2 L C_{p_0}$.
Under the assumption that the informativeness of observations decreases in higher levels we can replace the sample complexity of test $\mathcal{L}$ at the highest level and arrive at the first term in~\eqref{eq9}.
The second term in~\eqref{eq9}, is obtained by direct application of Lemma~\ref{SCofP} on the sample complexity of test $\mathcal{L}$ at the target node.

It remains to show that CBRW satisfies the reliability constraint. We know that at each visit of a leaf node the probability of declaring a non-target node as the target is lower than $\frac{\epsilon}{2LC_{p_0}}$ by the design of the test at leaf nodes. Thus, from the upper bound on the expected number of points visited by the random walk we have

\begin{eqnarray}\nn
\mathbb{P}[\mathcal{S}_\delta\neq\{(k_0,0)\}]&\le& 2L C_{p_0}\frac{\epsilon}{2L C_{p_0}}\\\nn
&=&\epsilon.
\end{eqnarray}

Order optimality of $\log K$ follows from information theoretic lower bound. Order optimality of $\log \frac{1}{\epsilon}$ can be established following standard techniques
in sequential testing problems.

An upper bound on the sample complexity of CBRW under hierarchical target setting can be obtained similarly. The trajectory of the
random walk can be analyzed by considering the last passage times $T_l$ of the random walk from subtrees $\mathcal{T}_l$
for $l=l_0+1,...,L$, as well as the last passage times $T'_1$ and $T'_2$ of the random walk from subtrees $\mathcal{T}'_1$ and $\mathcal{T}'_2$ which can be shown to not bigger than
$C^{H}_{p_0}$ following the similar lines as in the proof of Theorem~\ref{SCofRWCB} with

\begin{eqnarray}\label{CHP0}
C^H_{p_0}=\frac{1}{\bigg(1-\exp(-2(1-2(1-p_0)^3)^2)\bigg)^2}.
\end{eqnarray}

The analysis under hierarchical target setting differs from the analysis under leaf-level target setting in that the
consecutive calls of test $\mathcal{L}$ on the same node results
in increasing the confidence level.
We establish an upper bound on the expected total number $T_{\textrm{tot}}$
of observations from a node
at a series of consecutive calls of test $\mathcal{L}$ on the node
where the confidence level is divided by $2$ iteratively at each time.
Let $T^{(k)}$ be the number of samples taken at $k$'th consecutive call of
test $\mathcal{L}$ on a node. By design of CBRW strategy under hierarchical target setting the value of $p$ in test $\mathcal{L}$ is divided by 2 until the first time $k$ that $\frac{p_0}{2^k-1}<\frac{\epsilon}{3LC^H_{p_0}}$. Thus there are at most ${\lceil\log_2\frac{3LC^H_{p_0}p_0}{\epsilon}\rceil}$ consecutive calls of test $\mathcal{L}$ on one node. On a non-target node:

\begin{eqnarray}\nn
\hspace{-3em}\mathbb{E}[T_{\textrm{tot}}]
&\le& \sum_{k=1}^{{\lceil\log_2\frac{3LC^H_{p_0}p_0}{\epsilon}\rceil}} p_0^{k-1}\mathbb{E}[T^{(k)}]\\\nn
&\le&\sum_{k=1}^\infty p_0^{k-1}\bigg(\frac{48}{(\mu-\eta)^2}\log \frac{24\sqrt[3]{\frac{2^{k}}{p_0}}}{(\mu-\eta)^2}+2\bigg)~~~~\\\nn
&\le&\sum_{k=1}^\infty  p_0^{k-1}\bigg(\frac{48}{(\mu-\eta)^2}\log \frac{24\sqrt[3]{\frac{2}{p_0}}}{(\mu-\eta)^2}+2\bigg)\\\nn
&&~~~~~+\sum_{k=1}^\infty  p_0^{k-1}\frac{48}{(\mu-\eta)^2}\log\sqrt[3]{2^{k-1}}\\\nn
&=&\frac{1}{1- p_0}\bigg(\frac{48}{(\mu-\eta)^2}\log \frac{24\sqrt[3]{\frac{2}{p_0}}}{(\mu-\eta)^2}+2\bigg)\\
&&~~~~~+\frac{p_0}{(1- p_0)^2}\frac{16\log 2}{(\mu-\eta)^2}.
\end{eqnarray}

Upper bound on $\mathbb{E}[T_{\textrm{tot}}]$ in a conservative upper bound
on each single time that the test $\mathcal{L}$ is called.

On the target node:
\begin{eqnarray}\nn
\hspace{-3em}\mathbb{E}[T_{\textrm{tot}}]
&\le& \sum_{k=1}^{{\lceil\log_2\frac{3LC^H_{p_0}p_0}{\epsilon}\rceil}}\mathbb{E}[T^{(k)}]\\\nn
&\le&{{\lceil\log_2\frac{3LC^H_{p_0}p_0}{\epsilon}\rceil}}\bigg(\frac{48}{(\mu-\eta)^2}\log \frac{24\sqrt[3]{\frac{4}{\epsilon}}}{(\mu-\eta)^2}+2\bigg)\\
&\le&{\log_2\frac{6LC^H_{p_0}p_0}{\epsilon}}\bigg(\frac{48}{(\mu-\eta)^2}\log \frac{24\sqrt[3]{\frac{4}{\epsilon}}}{(\mu-\eta)^2}+2\bigg).
\end{eqnarray}
From the upper bound on $\mathbb{E}[T_{\textrm{tot}}]$, the upper bound on the sample complexity of
CBRW can be obtained. The satisfaction of the constraint on error probability can be shown similar to the leaf-level target setting.
\end{proof}

\end{document}